%% file: memauto.tex
\documentclass[twocolumn]{autart}

\usepackage{amsmath,amssymb}

\newcommand{\F}{\mathcal{F}}

\newcommand{\PA}{\mathcal{P}}
\newcommand{\ep}{\epsilon}

\newcommand{\LO}{\mathcal{L}}
\newcommand{\B}{\mathcal{B}}

\newcommand{\R}{\mathbb{R}}

\numberwithin{equation}{section}
\usepackage{color}
\usepackage{latexsym}

  \usepackage{natbib}
\begin{document}

\begin{frontmatter}
\runtitle{Memorization in generative models}

\title{A dynamical systems view of training generative models and the memorization phenomenon}

\thanks[footnoteinfo]{Corresponding author: V.\ S.\ Borkar.}

\author[ICTS]{Siva Athreya}\ead{athreya@icts.res.in},
\author[IISc]{Chiranjib Bhattacharya}\ead{chiru@iisc.ac.in},
\author[IITB]{Vivek S.\ Borkar}\ead{borkar.vs@gmail.com}

\address[ICTS]{International Institute for Theoretical Sciences,  Survey No.\ 151, Hesarghatta Road,   Shivakote Hobli, Bengaluru 560089, India}
\address[IISc]{Department of Computer Science and Automation, Indian Institute of Science, Bengaluru 560012, India}   
\address[IITB]{Department of Electrical Engineering, Indian Institute of Technology Bombay, Powai, Mumbai 400076, India}


\begin{abstract}
Using recent works of one of the authors (VSB) on collapse in generative models and two time scale dynamics in stochastic gradient descent in high dimensions, we give a system theoretic explanation of the memorization phenomenon in generative models. This relies purely on the dynamic aspects of the training phase. Specifically, we use a result of \cite{Austin} to motivate a stylized model for the loss function for stochastic gradient descent (SGD) wherein the loss function has a strong dependence on some variables and weak dependence on the rest in a precise sense. This naturally leads to two distinct time scales in the constant step size SGD that is commonly used in machine learning. This fact has been used to explain the double descent phenomenon in SGD in \cite{Borkar2}. In conjunction with a mathematical model for collapse phenomenon in SGD developed in \cite{Borkar1},  we analyze the constant step size SGD using the recent results of \cite{Azizian} in order to explain the phenomenon of \textit{memorization} wherein a generative model that is concurrently being tuned yields the same or similar outputs for significant stretches of time. This gives a novel perspective on the aforementioned phenomena reported in machine learning literature and their interrelationships, using a dynamical systems viewpoint.
\end{abstract}

\end{frontmatter}

\input{sec1intro}

\input{sec2techprem}

\input{sec3mainresults}
\section{Conclusions and future directions}\label{sec:discussion}

We have presented a dynamic view of the phenomenon of memorization
reported in the training of diffusion models, drawing upon some analysis of the related phenomenon
of `collapse' in machine learning. This also leads to some qualitative feel for the various factors affecting its occurrence and intensity, and throws up some interesting directions for further research. The work is grounded in some recent developments in stochastic approximation with constant step size due to \cite{Azizian} and flags its importance in the analysis of machine learning algorithms.

Going forward, one important problem is to estimate the error in the global minima of the Freidlin-Wentzell potential $V$ for the constant stepsize stochastic approximation, and those of  the function $\Gamma'(\cdot):= \Gamma(\lambda(\cdot),\cdot)$ being minimized. $\Gamma'$ would be the Freidlin-Wentzell potential for the Langevin dynamics and \textit{ipso facto}, an approximation for the Freidlin-Wentzell potential for its discretization, viz., the Langevin algorithm. $\Gamma'$ has an explicit control theoretic expression given by
$$\Gamma'(x) = \inf_{\mathcal{A}_x}\int_0^\infty \|u(t)\|^2dt,$$
where $\mathcal{A}_x$ is the set of all trajectories of the control system
$$\dot{x}(t) = \nabla F(x(t)) + u(t), t \geq 0,$$
that satisfy $x(0) = x$ and $x(t) \to$ one of the global minima of $F$, see \cite{FW}, Section 6.4,  \cite{Sheu}, \cite{BB1} for results in this vein for increasing level of generality.  The results of \cite{Azizian} indicates that an analogous expression should hold  for $V(\cdot)$. The respective potentials, being value functions of suitable control problems, should be unique viscosity solutions to the appropriate Hamilton-Jacobi equations in a suitable class of functions, e.g., continuous functions with appropriate bounds on their growth rates at infinity. This should facilitate a comparison between the two. This remains a research problem for the future.

\begin{ack}
SA was supported by  the Department of Atomic Energy, Government of India, under project no. RTI4001, CB was supported by research grant from Shell Markets India and  VSB was supported by a National Science Chair from the Government of India. A part of the work was done while VSB was visiting International Centre for Theoretical Sciences, Bengaluru. VSB thanks Dr.\ Anik Kumar Paul for helpful discussions regarding the material in the Appendix, which helped clarify some subtle issues.
\end{ack}

%\newpage
%\appendix
\vskip 0.2in

\bibliography{membib}
\newpage

\input{appendix}

\end{document}

%% file: sec1intro.tex
\section{Introduction}
An important issue with generative models such as diffusion models is that of `memorization' in their training process. This refers to the empirically observed phenomenon whereby it eventually starts regenerating the same or similar samples as before for a stretch of time. There is already a considerable body of work about this topic. For a small sampler, see \cite{Baptista,Buch,Chen,Gu,Hint,Kim,Pham,Wang,Wen,Ye}.
The objective of this article is to present a different take on this problem from a dynamical systems  viewpoint. This essentially means that we shift the focus  to the training algorithm, treating it as a dynamical system. The algorithm in question  broadly falls into the class of `stochastic approximation algorithms with Markov noise' \cite{BMP} (see also \cite{Borkarbook}, Chapter 8, for updates and a concise overview).  Using this fact and the known theoretical aspects of this class of algorithms, we look at  the phenomenon based on the temporal evolution of the training scheme. This leads to significant additional insights that are not apparent in a static view, which is usually based upon statistical mechanics models applied to the loss function landscape. In particular, we are not interested in the `thermodynamic limit' of statistical mechanical quantities as the domain of the loss landscape grows to infinity, in order to characterize critical phenomena. We are more interested in the asymptotic behaviour of the learning algorithm that explores this landscape. This explicitly incorporates the temporal aspects of learning in the analysis and facilitates additional insights into the observed phenomena.

This is not to say that this work  is the first work to take a dynamic viewpoint, but it does go deeper into the dynamics to give a clean picture of how the phenomenon of memorization emerges,  building upon a related but in some sense simpler phenomenon, viz., that of `collapse'. We argue that memorization emerges as a collapse-like phenomenon wherein a diffusion model, being a generative model, starts collapsing into a fixed output as in the pure collapse phenomenon, except that it is concurrently being modulated by the score learning scheme. This in conjunction with some other factors render this phenomenon intermittent. This also leads to some additional  intuition about what specifics of a given problem can encourage or discourage memorization. This ties together in a unified manner a variety of phenomena reported in machine learning.\\
%\tcr{Is the following the correct reading of the paper? One of the focus of the paper is to take a dynamic systems viewpoint to (1) recover existing results and (2) derive additional insights. In such a case an additional paragraph may help on what results can be obtained from a static viewpoint. Also a separate paragraph on problem (alluded to in the first contribution) maybe useful.}

A broad outline of our results is as folllows:\\

\begin{enumerate}
\item Based on an important result of \cite{Austin} recalled in section \ref{AustinThm}, we first observe that the problem structure, specifically the high dimensionality of the parameter space, naturally leads to multiple time scales in the learning algorithm. (See equations \eqref{fast3}-\eqref{slow3} below.) Two time scale dynamics is well studied in systems and control. This motivates a stylized multiple time scale model similar to the one used in \cite{Borkar2} in order to explain the temporal double descent phenomenon.  Specifically, the stochastic approximation scheme is now a constant step size two time scale stochastic approximation with Markov noise. The two temporal time scales in the training algorithm is an emergent phenomenon here due to the different spatial scales as observed in \cite{Borkar3}, and are not put there `by hand' as in the classical two time scale stochastic approximation schemes introduced in \cite{Borkar0}, \cite{Borkar3}.

\item We observe that the Markov chain in question comes from a generative model, which, in the classical set up, is known to exhibit the phenomenon of \textit{collapse}. Using the recent results of \cite{Borkar1}, which takes a somewhat different view of collapse, we explain the role of collapse in the present context in section \ref{Coll}. This  analysis is based on a dynamic interpretation of the onset of collapse in terms of an absorbing Markov chain, taken from \cite{Borkar1}. Specifically, a generative model leads to a probability measure-valued Markov chain that is absorbing, which is tantamount to saying that it asymptotically collapses to a Dirac measure concentrated at a random probability measure. This will be  our formal mathematical definition of collapse. See Theorem \ref{collap} below for a precise formulation.\\

\item We briefly recall a rigorous theory for constant step size stochastic approximation due to \cite{Azizian} in section \ref{role}. This is in fact grounded in the theory of stochastic dynamical systems with small noise as developed in \cite{FW}, but with a key difference that is highlighted later. This leads to a complete picture of the dynamic behaviour of the generative model tuning scheme, which includes some surprises as well. We push this theory a little further in order to give some additional insights about the asymptotic behaviour of constant step size SGD using some classical results of \cite{Hwang}. In particular, this sheds light on the observed preference of flatter minima by the SGD.

\item We then link this with the phenomenon of memorization in the context of generative  models. This is a more complex phenomenon than collapse due to additional dynamical effects that prevent a pure collapse. What this leads to is an intermittent collapse-like situation punctuated by regular evolution. The memorization phenomenon reported in the literature stands for a  generative model that is being trained, going through phases of repeating itself. We interpret memorization as repeated collapse with model drift on a slower time scale which renders such repetition intermittent. See Theorem \ref{memfin} below.

\item We then explain the role of noise in mitigating or eliminating the onset of memorization in this framework. This phenomenon has already been observed in the literature, most notably in \cite{Wu}, which we build upon. We explain this in the light of the theory for constant step size stochastic approximation \cite{Azizian}.  This leads to a complete picture of the dynamic behaviour of the generative model training scheme, which includes some surprises as well. 

\end{enumerate}

The article is organized as follows. Section 2 summarizes the key background material required in order to understand the subsequent developments, split into three sections. Section 2.1 recalls the standard formalism of diffusion models. Section 2.2 recalls a theorem of  \cite{Austin} that we use in order to motivate our formalism. Section 2.3, a critical one, recalls the two time scale stochastic approximation formalism in the context of stochastic gradient  descent (SGD for short). It is critical because the variant used here goes well beyond the standard framework thereof from \cite{Borkar0}, \cite{Borkar3} and subsequent works on the topic, in that the noise structure is much more complex. In fact there are some subtleties that have been overlooked in this respect in prior usage of this paradigm, which we underscore. They cannot be wished away and we opt to make a commonly used simplification in order to proceed. The issue is taken up again in the Appendix where some remedies from existing literature are highlighted.
Section 3 details a mathematical model  of the collapse phenomenon from \cite{Borkar3} which our model of memorization builds upon.

Section 4 is the core material of this article. Section 4.1 recalls the `vanilla' case of convergent two time scale SGD with decreasing step sizes, mimicking the treatment of \cite{Borkar3}. This is purely as a backdrop for what follows, in order to contrast with decreasing step size the situation arising in constant step size SGD that is our focus, and in order to highlight the fact that the phenomena which we try to explain are some emergent effects of \textit{constant step size SGD}, not present in the classical decreasing step size paradigm. As constant step size has its advantage and is the preferred choice in practice in the machine learning community, it is important that we clarify these subtleties. With this backdrop, section 4.2 analyzes the effects of constant step size in detail, clarifying and quantifying its asymptotic behaviour, building upon the important recent contributions of \cite{Azizian}. Section 4.3 then positions the memorization phenomenon in this context and explains its occurrence and non-occurrence in different empirical scenarios. As mentioned above, we revisit the subtleties about SGD with Markov noise of the kind we have and the difficulties in using the standard paradigms. We also propose in an Appendix as a remedy two schemes from literature that use a single function evaluation and analyze  one of them in detail to illustrate the issue.

%% file: sec2techprem.tex
Two main takeaways from this exercise that go beyond the specific technical contributions would be, first, highlighting the role of constant step size in the observed phenomena in SGD, and two, the subtleties of Markov noise and its algorithmic implications.

\section{Technical preliminaries}

This section is devoted to recalling the technical background for our main results. This is split into two distinct sub-sections. We begin with the simplest classical formulation of diffusion models as our underlying template for a tunable generative model in sub-section \ref{DM}.  The next sub-section recalls an important result of \cite{Austin} that motivates our basic formulation, which we describe  in sub-section \ref{style}. We follow this in sub-section \ref{MNoise} with a short intuitive tutorial on two time scale stochastic approximation with the so called Markov noise. We only state the variant relevant for our purposes, viz., stochastic gradient descent, in order to not cloud the main results with excessive technicalities.\\

\noindent
{\bf Notation:} Recall that a Polish space $S$ is a  Hausdorff topological space that is separable (i.e., has a countable dense set) and admits a compatible complete metric (i.e., it has a metric $d: S \times S \to \R$ satisfying: 1.\ $x_n \to x_\infty$ in $S$ if and only if $d(x_n,x_\infty) \to 0$, and, 2.\ if $d(x_n,x_m) \to 0$ as $m,n\to \infty$, then there exists an $x_\infty \in S$ such that $x_n \to x_\infty$. The uniqueness of this limit is easy to prove.) Let $C_b(S)$ denote the space of bounded continuous functions $S \to \R$. Likewise, $C([0,T];\R^d)$ will denote the Banach space of continuous functions $[0,T]\to\R^d$ with the maximum norm. We shall denote by $\PA(S)$ the space of probability measures on $S$ with the coarsest topology that renders continuous the maps $\mu \in \PA(S) \mapsto \mu(f) := \int fd\mu \in \R$ for all $f \in C_b(S)$.  This space is Polish with one choice of a complete metric being the Prohorov metric (\cite{Bill}). Polish spaces turns out to be the correct level of generality for most applications of probability theory because most of the standard theory of convergence of probability measures (among other things) extends to Polish spaces and not beyond. Also, Polish spaces include  most function spaces of importance in stochastic processes, and even more importantly, because Polish spaces of probability measures on Polish spaces are Polish, a blessing probabilists appreciate only too well. For a Polish space $S$, the law of an $S$-valued random variable $X$ (say) will be denoted by $\LO(X) \in \PA(S)$. 

We use the abbreviation `ODE' for `Ordinary Differential equations'. Also, by `a.s.', an abbreviation for `almost surely', we mean `with probability $1$', following the standard usage in probability theory.\\

\subsection{Diffusion models}\label{DM}

We describe here the classical diffusion model \cite{Song} (see \cite{Yang} for a comprehensive survey of early work on this topic). We confine this brief introduction to the simplest model, which is the stable linear diffusion process, also called \textit{Ornstein-Uhlenbeck process}. This is  given by a linear stochastic differential equation in $\R^d$ :
\begin{equation}
dX(t) = - \upsilon X(t)dt + dW(t), \ 0 \leq t \leq T. \label{OU}
\end{equation}

Suppose we consider the time-reversed process $Y(t) = X(T-t), t \in [0,T]$, and simulate it for $t \in [0,T]$. Then one would get at time $0$ a noisy version of $X(0)$. This can be taken to be a sample from the distribution from which $X(0)$ was drawn. This is the standard interpretation of this procedure.

The difficulty is in reversing the diffusion. It can be shown \cite{Anderson,Follmer,Haussmann} that $Y(t), t \in [0,T]$, satisfies the stochastic differential equation
\begin{eqnarray}
dY(t) &=& (Y(t) + \nabla \log p(Y(t)|X(0),t))dt + W(t), \nonumber \\
&& \ \ \ \ \ \ \ \ \ \ \ \ \ \ \ \ \ \ \ \ \ \ \ \ \ \ \ \ \ \ \ \ \ \  t \in [0,T],  \label{OUR}
\end{eqnarray}
where $p(\cdot|X(0),t)$ is the conditional probability density of $X(t)$ given $X(0)$ and $\nabla $ denotes the gradient operator. The problem then is to learn this `score function' $(x,t) \mapsto \nabla \log p(\cdot|x,t)$. This is to be learnt by minimizing the empirical mean square error
\begin{equation}
F(\theta) := \widetilde{E}\left[\|s_\theta(X(t), t) - \zeta(t)\|^2\right], \label{scoreError}
\end{equation}
where $s_\theta(\cdot, \cdot)$ is a parametrized approximation of the score function, usually a deep neural network (DNN) parametrized by a vector parameter $\theta \in \R^d$, and $\zeta(t)$ is a sample based estimate  of the score function $\nabla\log p(X(t)|X(0),0)$. $\widetilde{E}[\cdots]$ denotes the empirical mean of $(X(0), X(t))$ over samples, i.e.\ the sample average.\\

The phenomenon we are interested in is the empirically observed fact that the diffusion model, modulated by the score function that is being concurrently learned in a suitably parametrized form,  tends to output the same image or very similar images for significant stretches of time. This is called \textit{memorization}.

While flagged in particular in the context of diffusion models, this formalism underlies all generative models that are being tuned by a training algorithm. So we abstract out the problem placing it in the context of an abstract and idealized generative model borrowed from \cite{Borkar1}. Then in order to keep matters (in particular, the notation) simple, we also abstract out the underlying stochastic approximation algorithm with Markov noise and analyze it as such in subsection \ref{MNoise}. We  return to the implications for generative models  in section 4.

\subsection{Austin's theorem}\label{AustinThm}
 
A recent work \cite{Borkar2} proposed a dynamic view of `double descent' in SGD based on a stylized model which we borrow here. This model uses the fact that given a bound on the Lipschitz constant of a Lipschitz function of a very large number of parameters, the function will depend mostly on a subset of the parameter set. Something of a folklore in the machine learning community, this is made precise in the recent work of \cite{Austin}. We state Austin's theorem below.

Let $(X, d)$ be a compact metric space and $\mu$ a probability measure on its Borel $\sigma$-algebra whose support is connected and locally connected. For  $N \geq 2$,
define the metric $d^N(\cdot,\cdot)$ on $X^N := X\times \cdots \times X$ ($N$ times) by:
\begin{eqnarray*}
\lefteqn{d^N([x_1,\cdots,x_N],[y_1,\cdots,y_N]) := max_{1 \leq i \leq N}d(x_i,y_i)} \\
&& \ \ \ \ \ \ \ \ \ \ \ \ \ \ \ \ \ \ \ \ \ \ \ \ \ \ \ \ \ \ \ \ \ \ \ \ \ \ \ \ \ \ \ \ \ \ \ \ \ \ \ \ \ \ \ \ \ \  < \infty.
\end{eqnarray*}
Let $\omega : [0,\infty) \to [0,\infty]$ be a continuous non-decreasing `\textit{modulus of continuity}'. 
Define
$\Gamma_S(\cdot) :=$ projection $L_2(\mu) \to$ its subspace corresponding to coordinates  $i \in S \subset \{1,\cdots,N\}$.

\begin{thm}\label{Austin} (Austin's theorem) For every $\ep > 0$,
there exists an integer $p \geq 1$ depending on $X, \ep$ and $\omega$, with the following property:    For every $N > 1$, if $f : X^N \to \R$ has modulus of continuity at most $\omega$
for the metric $d^N$, then there exists an $S \subset [1,\cdots,N]$ with $|S| \leq p$ such that
$$\|f - \Gamma_S(f)\|_{L_1(\mu^N)} < \ep.$$
\end{thm}

In other words, a hard constraint on the modulus of 
continuity forces the function to depend mostly on a subset of variables, in a precise sense.
Note that modulus of continuity $= 1$ corresponds to Lipschitz continuity, which is the typical case of interest in machine learning because input-output maps of standard feedforward networks do satisfy this condition. What the foregoing then implies is that a hard bound on the Lipschitz constant of an overparametrized DNN implies that the input-output map depends mostly on a subset of parameters in a precise sense, something of a folklore in machine learning.

\subsection{A stylized model}\label{style}

% The stylized model we propose (brrowd from \cite{Borkar2}) is as follows. For some $0 < \ep \ll \frac{1}{2}$, assume that $F$ has the form
%\begin{equation}
% F(z) = f(x, \ep y). \label{scales}
%\end{equation}
In \cite{Borkar2}, the above observation was used to motivate a stylized  model for explaining phenomena such as grokking or double descent, as follows. 
 Write $z = (x,y)\in \R^d$ where $x \in \R^s, y \in \R^r$ for some $r,s\geq 1$ with $s+r=d$. Assume that the objective function $F: \R^d \to \R$, $d\gg 1$, has the form :
for some $0 < \ep \ll \frac{1}{2}$,
\begin{equation}
 F(z) = f(x, \ep y). \label{scales}
\end{equation}
  This captures the relative strength of the dependence of $F$ on $x$ and $y$. Specifically, $F$ depends significantly more on $x$ than on $y$. We shall adopt this paradigm here.

\begin{rem} As pointed out in \cite{Borkar2}, this clear cut separation of spatial scales is a indeed stylized model. The actual case can be much more complex, e.g., there can be more than two scales or a continuum of scales. For the analysis here, it suffices to have two \textbf{dominant}  scales. \end{rem}

\subsection{The effect of Markov noise on SGD}\label{MNoise}

In this sub-section, we summarize what is perhaps the most technical part of the background material required for our analysis, viz.\ SGD with Markov noise. Let $Z_n$ for each $n$ represent the trajectories generated in a single forward-backward pass of the generative model, with state space given by a Polish space $S$ and with transition probabilities modulated by the parameters of the parametric candidate score function that is being tuned.. For example, for the classical diffusion model described above, $S = C([0,T];\R^d)^2$. 

Let $(x,y,w) \in  \R^s\times\R^r\times S \mapsto f(x,y,w) \in \R$ denote the empirical loss function. As before $z = (x,y)$, also, $z_n = (x_n,y_n)$. Consider the SGD
\begin{equation}
z_{n+1} = z_n- a\nabla_zf(z_n, Z_n), \label{SGDor}
\end{equation}
where $\{Z_n\}$ is a \textit{Markov noise} in stochastic approximation parlance \cite{BMP}, Chapter 8 of \cite{Borkarbook}. That is,
\begin{eqnarray}
\lefteqn{P(Z_{n+1} \in A | z_0, Z_m, m \leq n) =} \nonumber \\
&& \ \ \ \ \ \ \ \ \ \ \ \ \ \ \ \ \ \  p_{z_n}(A|Z_n) \ \ \forall \ A  \ \mbox{Borel}  \subset S, \ n \geq 0. \label{p-ex}
\end{eqnarray}
Here $p_z(\cdot|\cdot)$ is a transition kernel on $S$ parametrized by $z$. Assume that it is Lipschitz in $z$, uniformly with respect to the other arguments.\\

Letting $\nabla_1, \nabla_2$ denote the partial gradients w.r.t.\ $x,y$ resp., the SGD \eqref{SGDor} splits into
\begin{eqnarray}
x_{n+1} &=& x_n - a\nabla_1f(x_n, \ep y_n, Z_n), \label{fast} \\
y_{n+1} &=& y_n - \ep a\nabla_2f(x_n, \ep y_n, Z_n). \label{slow}
\end{eqnarray}

Note that we have replaced $\nabla_x, \nabla_y$ by $\nabla_1, \nabla_2$ resp., to indicate that the gradient is with respect to the first and the second argument respectively. The need to do so will soon become apparent.
Note also that \eqref{slow} has the effective stepsize of $b := \ep a \ll a$. This allows \eqref{fast} to treat the slowly varying $y_n$ as essentially a constant parameter, i.e.\  $y_n \equiv y$. 

Using the theory of stochastic approximation with Markov noise, we consider instead, modulo small (i.e.\  $o(a)$) errors, the related dynamics
\begin{equation}
x_{n+1} = x(n) - a\left(\int\nabla_1f(x_n, \ep y_n, z)\pi_{(x_n,y_n)}(dz)\right), \label{fast2} 
\end{equation}
\begin{equation}
y_{n+1} = y(n) - \ep a\left(\int\nabla_2f(x_n, \ep y_n, z)\pi_{(x_n,y_n)}(dz)\right). \label{slow2}
\end{equation}
The intuition, which has been made rigorous in literature \cite{BMP}, \cite{Borkarbook}, is that $\{Z_n\}$ evolves on the `natural' clock $n \geq 0$ which is much faster than the `algorithmic' clock $na, n \geq 0$, of $\{(x_n,y_n)\}$ for $a \ll 1$. Hence $\{(x_n,y_n)\}$  see only the average effect of the $\{Z_n\}$, i.e.\ $Z_n$ contribution averaged out with respect to its stationary distribution $\pi_{(x_n,y_n)}$ for slowly varying $(x_n,y_n)$. 

\begin{rem}\label{grad} The notation $\nabla_1, \nabla_2$ now makes it clear that the gradients do not account for the additional dependence on $x_n,y_n$ in their role as subscripts of $\pi_{\cdots}(dz)$. Hence the terms in parentheses in \eqref{fast2}-\eqref{slow2} are not necessarily gradients. 
Note, however, that this observation applies only when an actual gradient is used. In most cases, one uses an approximate gradient that requires only function evaluations, but no explicit differentiation. Examples are classical finite difference methods such as the Kiefer-Wolfowitz algorithm \cite{KW}, simultaneous perturbation stochastic approximation (SPSA) \cite{Spall}, methods based on local regression \cite{Mukherjee} or on stochastic perturbation \cite{Flaxman}. These schemes are not free of  problems, but these problems  are of a different kind. For the time being, we ignore them  and revisit the matter in the Appendix. \end{rem}

In view of Remark \ref{grad}, we use the coupled iterations
\begin{equation}
x_{n+1} = x(n) - a\nabla_1\left(\int f(x_n, \ep y_n, z)\pi_{(x_n,y_n)}(dz)\right), \label{fast3} 
\end{equation}
\begin{equation}
y_{n+1} = y(n) - \ep a\nabla_2\left(\int f(x_n, \ep y_n, z)\pi_{(x_n,y_n)}(dz)\right), \label{slow3}
\end{equation}
which are assumed to hold approximately (i.e.\ with $o(a)$ errors). The difference with \eqref{fast2}-\eqref{slow2} is that the gradients with respect to $x,y$ resp.\ now are over both occurrences of these variables, i.e., they also account for these variables where they appear as subscripts of $\pi_{(\cdot,\cdot)}$. .

\begin{rem} We started with a stylized model where $z = (x,y)$ and $F(z) = f(x, \ep y)$. In practice, even if we assume a clear separation of time scales implicit here, it is possible that multiple models with this property fit the data. In fact the split $z = (x,y)$ and the concomitant representation $F(z) = f(x, \ep y)$ may be an \textit{emergent} phenomenon during the iteration. In that case, the above analysis applies post any such emergence. More on this in section \ref{role}.
\end{rem}

\section{The collapse phenomenon}\label{Coll}
In this section we shall explain memorization by relating it to another, simpler phenomenon called `collapse'. 
There has been a considerable interest and research in machine learning in the phenomenon dubbed as `collapse'  of generative models, see \cite{Alem,Elvis1,Elvis2,Gerst,Marchi,Shum1,Shum2,Suresh}  for a sampler. This list in fact represents a broad spectrum of related phenomena. We summarize here the formalism of a particularly simple case studied in \cite{Borkar1}, as it captures the essence of the phenomenon is precise mathematical terms, relating it to the dynamics of an absorbing Markov chain in the space of probability measures. While an idealization, it allows us to understand the underlying dynamics and then treat observed manifestations thereof as a perturbation of this stylized model. . 

Consider a Polish space $S$ and a prescribed $\mu_0 \in \PA(S)$. The generic generative model we shall consider is as follows. Let  $a \in [0,1]$. Let $P_\theta$ denote a prescribed parametric family in $\PA(S)$  for $\theta \in$ a closed  parameter set $\Theta$ of some Polish space, with $P_{\theta_0} = \mu_0$ for some $\theta_0 \in \Theta$. We allow nonparametric fits by allowing the possibility that $\Theta = \PA(S)$. At each $n \geq 0$, proceed as follows.

\begin{enumerate}

\item We generate $N \gg 1$ $S$-valued i.i.d.\ random variables $\{X^i_n, 1 \leq i \leq N\}$ with their common law being given by  $a\mu_0 + (1 - a)\mu_n$.\\

\item We fit a distribution $P_{\theta(n+1)}$ to $\{X^i_n, 1 \leq i \leq N\}$ (the classical case uses the empirical distribution).\\

\item We assume that this is done in a way such that $\mu_n$ is the \textit{conditional barycenter} of $\mu_{n+1}$ for all $n$, in the sense that, w.\ p.\ 1,
\begin{equation}
E[\mu_{n+1}(A)|X^i_m, 1 \leq i \leq N, 0 \leq m \leq n] = \mu_n(A) \label{bary}
\end{equation}
for all $A \in \B(S)$. Equivalently,
\begin{eqnarray}
\lefteqn{E\left[\int fd\mu_{n+1}\Big|X^i_m, 1 \leq i \leq N, 0 \leq m \leq n\right]} \nonumber \\
&& \ \ \ \ \ \ \ \ \ \ \ \ = \int fd\mu_n \ \ \  \forall f \in C_b(S).  \label{bary2}
\end{eqnarray}
Let $\zeta_n \in \PA(\PA(S))$ denote the law of $\mu_n$ for $n \geq 0$.

\end{enumerate}

Note that \eqref{bary} is always true in the case of successively generated empirical distributions mentioned above. Another situation where it is true is some schemes wherein one considers a parametric family parametrized by $\theta$ (say) and updates the conditional law of the true $\theta$ using Bayes rule (\cite{Suresh}). 
In general, however, this  assumption is an approximation. We shall return to this issue later. 

The main results of \cite{Borkar1} are summarized as follows.

\begin{thm}\label{collap} $(i)$ $\{\mu_n\}$ forms a time homogeneous  Markov chain in $\PA(S)$ such that $\mu_n$ is the conditional barycenter of the law of $\mu_{n+1}$ for all $n$. 

\noindent $(ii)$ For $a = 0$, $\mu_n \to \delta_\gamma$ w.\ p.\ $1$ for a random $\gamma \in S$.
\end{thm}

 Part $(ii)$ above follows from the simple observation that if a time-homogeneous Markov chain converges a.s., it must converge to an absorbing state. It captures mathematically the commonly understood notion of \textit{collapse}. We take this as our formal definition of collapse. In particular, this is the definition we use in this article. Barycenter of the random Dirac measure $\delta_\gamma$ remains $\mu_0$.
 
 The result of Theorem \ref{collap} $(i)$ is better viewed as \textit{degeneration} rather than collapse as argued in \cite{Borkar1}, because the samples can be viewed as noisy versions of $\mu_0$, worse for smaller values of $a$. 

As already mentioned, we shall view memorization as an intermittently repeated manifestation of the collapse phenomenon.

%% file: sec3mainresults.tex
\section{Main results}
%\tcr{This section needs to be throughly proofread. Connection to diffusion model seems to be weak. Any specialized statements we can provide for equation \eqref{expl} and related. Otherwise the equations do not seem to be used.} 
We begin with a brief overview of the convergent case of two time scale SGD in sub-section \ref{converg} as a backdrop, using decreasing step sizes. This is in order to flag and underscore  the observed phenomena, both the desired  and the anomalous ones, as consequences of constant step size, which is the task undertaken in sub-section \ref{role}. Subsection \ref{onset} then explains how the occurrence of memorization or the lack thereof comes about. All this will be in the framework of two time scale SGD with Markov noise introduced above, which is our abstraction for the score function learning scheme, with Markov noise provided by the generative model being tuned. The classical diffusion model provides a template for the same.

\subsection{The convergent case}\label{converg}
%\tcr{I guess this section is about the convergence of the model. Title can be used to reflect that}
We now return to two time scale stochastic approximation, but with an important  tweak, viz., the use of decreasing step sizes. This is a simpler case to analyze and helps motivate the case of interest here, i.e.\ constant step size, later on. Suppose the step sizes $a$ and $\ep a$ in the foregoing are replaced by decreasing stepsizes $\{a_n\}, \{b_n\} \in (0,1)$ satisfying the Robbins-Monro conditions 
$$\sum_na_n = \sum_nb_n = \infty, \ \sum_n(a_n^2+ b_n^2) < \infty,$$
along with the time scale separation
$$b_n = o(a_n).$$

Then the classical two time scale stochastic approximation theory (\cite{Borkarbook} Section 8.1) leads to the following conclusions. We only sketch the arguments, which closely follow the classical analyssis of two time scale stochastic approximation \cite{Borkar3}
%\tcr{The proof can be more elaborate. Since this is one of the main results. completeness will be good--asymptotically reduce to. Does it mean for large $n \ge N_0$?}

\begin{thm} The iterations \eqref{fast3}-\eqref{slow3} can be expressed as
 \begin{equation}
x_{n+1} = x_n - a_n\nabla_1\Gamma(x_n, y_n) +o(a_n), \label{fast4} 
\end{equation}
\begin{equation}
y_{n+1} = y_n - b_n\nabla_2\Gamma(x_n, y_n) + o(b_n), \label{slow4}
\end{equation}

for a suitably defined $\Gamma: \R^d \to \R^d$.
\end{thm}

\begin{pf} (Sketch) This proof follows by standard theory of two time scale stochastic approximation We only sketch the arguments here at an intuitive level, referring the reader to sections 8.2-8.3 of \cite{Borkarbook} for details. Observe that in the stochastic gradient descent \eqref{fast3}-\eqref{slow3} , there are \textit{three} time scales:\\

\noindent 1.\  The `Markov noise' $\{Z_n\}$ taking values in $S$ evolves on the `natural' clock of $n = 0,1,2,\cdots$, which is the fastest of the three time scales. \\

\noindent 2.\ As $a_n \ll 1$ for large $n$, \eqref{fast3} moves on a \textit{medium} time scale $\sum_{m=0}^na_m, n = 0,1,2,\cdots,$ which is slow relative to natural, i.e., the \textit{fast} time scale. \\

3.\ In turn, because $b_n = o(a_n)$, \eqref{slow3} evolves on the slowest time scale $\sum_{m=0}^nb_m, n = 0,1,2,\cdots$, dubbed the \textit{slow} time scale.\\

Then $\{Z_n\}$ sees $z_n = (x_n,y_n)$ as quasi-static, i.e., $(x_n,y_n) \approx (x,y)$, whence Theorem \ref{collap} applies and $\LO(Z_n) \to \delta_{\gamma(z)}$ in $\PA(S)$ for a possibly random $\gamma(z) \in S$. Since this happens on the natural time scale, the iterates $(x_n,y_n)$ on the slower time scales see $\{Z_n\}$ as $\approx \delta_{\gamma(x_n,y_n)}$. Letting $\Gamma(z) = \Gamma(x,y) := f(x,y,\gamma(x,y))$, we have \eqref{fast4}-\eqref{slow4}.
\end{pf}

This is a two time scale gradient descent \cite{Borkar3}. Suppose that 
$$\lambda(y) := argmin_x \Gamma(x,y)$$ 
is uniquely defined and Lipschitz, and $\Gamma(\cdot, y)$ has no other critical points. Also, suppose further that $\Gamma(\lambda(\cdot),\cdot)$ has a unique minimum $y^*$. Then this reduces to the special case of \cite{Borkarbook}, Section 8.1, which precisely says that $\lambda(x_n) - y_n \to 0$ a.s.\ in the more general context of two time scale stochastic approximation (not only SGD).

\begin{thm} Suppose $\{Z_n\}$ is being generated by a pure generative model as above with parameters $(x_n,y_n), n \geq 0,$ being tuned by \eqref{fast}-\eqref{slow}. Then asymptotically, $\{Z_n\}$ exhibits collapse. 
\end{thm}

\begin{pf} By the standard analysis of two time scale algorithms, $x_n \approx \lambda(y_n)$ and therefore modulo some asymptotically negligible error (i.e., $o(b_n)$) terms, 
$$y_{n+1} = y_n - b_n\nabla\Gamma(\lambda(y_n), y_n) + o(b_n).$$
Recall that $\lambda(y) = \mbox{argmin} (\Gamma( \cdot, y))$ Thus it follows from  Danskin's theorem (\cite{Danskin}) that $\nabla\Gamma(\lambda(y),y) = \nabla \min_x\Gamma(x,y_n)$, making it an SGD . Now $y_n \to y^*$ (say). Then $x_n \to \lambda(y^*)$ and $\{Z_n\}$ is asymptotically a pure Markov chain generative model that will exhibit pure collapse by Theorem \ref{collap} $(i)$.  
\end{pf}

In case of multiple local minima of $\Gamma(\cdot, y)$, we get convergence to one of the local minima if the noise is `rich enough' in all directions in a suitable sense, so that the iterates get pushed away from equilibria other than local minima, a.s. This is the `\textit{avoidance of traps}' by stochastic approximation algorithms, see \cite{Borkarbook}, Section 3.4, and the references therein. As this is standard fare of the classical stochastic approximation theory, we skip the details.

However, we are \textit{not} using decreasing step sizes. We argue in the next subsection that memorization as it is known in practice is explicitly a phenomenon caused by the fact that we use a constant step size.  

\subsection{The role of stepsizes}\label{role} 

A very detailed mathematical analysis of the asymptotic behaviour of the classical SGD in $\R^d$ with constant step size, given by
\begin{equation}
x_{n+1} = x_n - a(\nabla f(X_n) + M(X_n)), \label{baseSGD}
\end{equation} 
can be found in \cite{Azizian}. We take it as a starting point and make a precise quantitative statement about how the limiting distribution concentrates on global minima of the large deviations rate function in the $a\downarrow0$ limit. We first summarize  below some qualitative statements taken verbatim from \cite{Azizian}. Here $M(x)$ is a family of random variables that is zero mean with uniformly positive definite covariance matrix and sub-gaussian tails, twice continuously differentiable in $x$, with $\|M(x)\|$ satisfying a uniform linear growth condition in $\|x\|$. 

\begin{rem}\label{NoiseModel} The above model correctly captures the standard scenario in SGD as applied in machine learning. Suppose $\ell(\theta, X_n, Y_n)$ denotes the empirical loss at time $n$, parametrized by the parameter $\theta$. The gradient descent will use the instantaneous gradient $\nabla_\theta \ell(\theta_n,X_n,Y_n)$ where $\{(X_n,Y_n)\}$ are i.i.d., $\nabla_\theta$ denotes the gradient in the $\theta$ variable,  and $\theta_n$ is the current parameter estimate. Replacing $X_n$ in \eqref{baseSGD} by $(X_n,Y_n)$ and setting 
$$f(\theta) :=  E[\ell(\theta,X_n,Y_n)],$$
$$ \ M(X_n,Y_n) :=  \nabla_\theta \ell(\theta,X_n,Y_n) - \nabla f(\theta),$$ 
we get \eqref{baseSGD}.\end{rem}

The key conclusions of \cite{Azizian}, taken verbatim from there, are:

\begin{enumerate}
\item \textit{In the long run, the critical region of $f$ is visited exponentially more often than any
non-critical region.}

\item \textit{The iterates of SGD are concentrated with exponentially high probability in the
vicinity of a region that minimizes a certain ``energy functional" which depends on
$f$ and the statistics of the noise in SGD. \textbf{Importantly, the ground state of this
functional does not necessarily coincide with the global minimum of $f$.}}

\item \textit{Among the remaining connected components of critical points, each component is
visited with frequency which is exponentially proportional to its energy, according to
the Boltzmann–Gibbs distribution of statistical physics with temperature equal to
the method’s step-size.}

\item \textit{Every connected component of non-minimizing critical points of $f$ – i.e., local maximizers
or saddle points – is ``dominated" by a component of local minimizers that is
visited exponentially more often.}
\end{enumerate}

\medskip

This applies to our case with `negative energy' in place of `energy' in bullet 3. The analysis of \cite{Azizian}, which is too extensive and intricate to include here in toto, is inspired by the analysis of \cite{FW} for small noise limit of diffusion processes. 

\begin{rem} In some recent works \cite{Niao}, \cite{Mandt}, the Freidlin-Wentzell theory is directly invoked to claim that the stationary distribution of \eqref{baseSGD} concentrates on the global minima of $F$. This is based on viewing the noise on the time scale of stepsize $\sqrt{a}$, which allows us to view the constant step size SGD as a discrete analog of the Langevin diffusion
$$dx(t) = -\nabla F(x(t))dt + dW(t),$$
where $W(\cdot)$ is a standard Wiener process in $\R^d$. However, it is clear from Remark \ref{NoiseModel} that the noise should be scaled by $a$ and not by $\sqrt{a}$. This is precisely what leads to the comment in boldface in the second bullet in the remarks from \cite{Azizian} cited above. We comment on this issue again in the concluding section.
\end{rem}

We next refine the above analysis in view of the results of \cite{Hwang}, which allows us to explain another phenomenon observed in practice - SGD prefers simpler (read `flatter') minima. Let $V(\cdot)$ denote the rate function for the large deviations for the $a\downarrow 0$ limit in the stationary distribution for constant stepsize SGD established in \cite{Azizian}. Assume that $V(\cdot)$ is twice continuously differentiable and has finitely many isolated global minima at $m_i, 1 \leq i \leq \ell$, where $\nabla^2V$ is non-singular. 

\begin{thm} As $a\downarrow 0$, the asymptotic probability weights at $m_i, 1 \leq i \leq \ell$, are given by $\left(\prod_{j=1}^d\Lambda_j^k\right)^{-\frac{1}{2}},$
where $\Lambda^k_j > 0, 1 \leq j \leq d,$ are the eigenvalues of $\nabla^2V(m_k)$.
\end{thm}

\begin{pf}
As $a\downarrow 0$, the stationary distribution becomes proportional to $e^{-\frac{V(x)}{a}}$ and thus concentrates on the global minima of this `potential' $V(\cdot)$ by the results of \cite{Azizian}.
Suppose $V$ is twice continuously differentiable and its Hessian $\nabla^2V$ is nonsingular at each global minimum $m_i, 1 \leq i \leq \ell$. Then by a result of \cite{Hwang}, the asymptotic probability mass at the global minimum $m_k$ (say) is proportional to 
$$\left(det\left(\nabla^2V(m_k)\right)\right)^{-\frac{1}{2}}  \ = \ \left(\prod_{j=1}^n\Lambda_j^k\right)^{-\frac{1}{2}},$$
where $\Lambda^k_j > 0, 1 \leq j \leq m,$ are the eigenvalues of $\kappa(m_k) := \nabla^2V(m_k)$.
\end{pf}

Thus the small noise limit favours the global minima $m_i$ where $\det(\nabla^2V(m_i))$ is low, whence not all eigenvalues can be high. 
This suggests that only few eigendirections matter, supporting the intuition that simpler minima that afford a compact description in the local eigenbasis are favoured.

\subsection{The onset of memorization}\label{onset}

We now combine the foregoing to explain the intermittent phenomenon of memorization that occurs in the  training of generative models. We saw in subsection \ref{Coll} that in the decreasing step size case with convergent iterations on both time scales, a collapse would occur even with concurrent tuning of the parameters of the generative model. The catch is that with constant step size and multiple equilibria, we do not expect convergence. Focusing on the fast time scale alone, an isolated minimum of $\int f((\cdot, y), z)\pi_{(\cdot,y)}(dz)$ would be parametrized by $y$, so if $y$ drifts on a smaller time scale, the minimum over $x$ may drift and / or jump, as in, e.g., the well known relaxation oscillations in mechanics (\cite{Gentz}). Even \textit{within} a single time scale, a gradient descent can jump between multiple stable equilibria, though only rarely as the results of \cite{Azizian} show. How rare is rare enough will depend on the magnitude of $a$. 

We next give a stylized picture of memorization in a simplified framework.
Consider the case when \eqref{fast3}-\eqref{slow3} represent the SGD for training a diffusion model.  Let
$$\varphi(x,y) := \int f(x, \ep y, z)\pi_{(x,y)}(dz).$$
Suppose for each $y\in\R^r$, $\varphi(\cdot,y)$ has multiple isolated local minima labelled $\lambda_i(y), 1 \leq i \leq m,$ where $\lambda_i$'s are locally Lipschitz. Suppose also that $\varphi(\lambda_i(y), y)$ has multiple isolated local minima, labelled $\hat{y}_1, \cdots, \hat{y}_k$ resp. Assume that $\ep \ll 1$.

\begin{thm}\label{memfin} Suppose $\hat{y}_i \neq \mbox{argmin}(\varphi(\cdot, \hat{y}_i))$ for all $i, 1 \leq i \leq k$. Then the  dynamics \eqref{fast3}-\eqref{slow3} can exhibit memorization in the sense that $x_n - \lambda(y_n) \approx 0$ for a $\lambda: \R^r \to \R^s$ for long stretches of $n$, with $\{y_n\}$ evolving on a much slower time scale than $\{x_n\}$.
\end{thm}

\begin{pf} From standard stochastic approximation theory (specifically, the `ODE approach', see Chapter 2 of \cite{Borkarbook} or \cite{BMP}), $\{(x_n,y_n)\}$ track the  the coupled ODEs 
\begin{eqnarray}
\dot{x}(t) &=& -\nabla_1\varphi(x(t),y(t)),  \label{fast5}\\
\dot{y}(t) &=& -\ep\nabla_2\varphi(x(t),y(t)). \label{slow5}
\end{eqnarray}
Consider now \eqref{fast2}-\eqref{slow2} as representing our scheme for tuning a generative model such as the diffusion model, with $\{Z_n\}$ as the input coming from the generative model. Then because $\ep \ll 1$, \eqref{fast} sees $y_n$ as quasi-static and starts tracking $\lambda(y_i)$ for some $1\leq i \leq k$. Then, in a small neighbourhood $O(y_i)$ of $y_i$,  $\{y_n\}$ starts tracking the ODE 
$$\dot{y}(t) = -\nabla \varphi(\lambda(y(t)),y(t)),$$
where we use Danskin's theorem as before. However, since $y_i \notin \mbox{argmin}\varphi(\cdot, y_i)$, the vector field $-\nabla \varphi(\lambda(y(t)),y(t))$ remains bounded away from zero in $O(y_i)$, whence $y(t)$ exits $O(y_i)$. Because of the presence of $\ep$ on the right hand side of \eqref{slow4}, this movement is slow and $x(t)$ does not change much, thus leading to a non-negligible time when $x(t) \approx \lambda_i(t)$. Since $(x_n,y_n)$ track $(x(t),y(t))$ on a moving time window, $(x_n,y_n)$ mimic this behaviour, leading to memorization. 
\end{pf}

For $k = 2$, one would get the well known relaxation oscillations phenomenon in mechanics already alluded to above (\cite{Gentz}). Here the trajectory spends substantial time in the neighbourhoods of two points, occasionally making a relatively rapid transition from one to the other. This is distinct from the occasional jump between multiple equilibria in the frameworks of \cite{Azizian} or \cite{FW} - as already noted, these are on an exponentially small time scale, i.e., rare. Memorization, when it occurs, is intermittent, but not rare. This is because the transitions are not purely noise-driven, but aided by a drift in the dynamics itself on a slower albeit non-negligible time scale.

Even so, there were many idealizations in the foregoing, most prominently in assuming the conditional barycenter property of the random probability distributions generated by the generative model.  The probability measure $P_{\theta(n)}$  need not be the exact conditional barycenter of $P_{\theta(n+1)}$ for $n \geq 1$. We have already observed that the conditional barycenter property holds for generative models that use successive empirical distributions based on simulated samples and in some idealized bayesian frameworks. Also, while the schemes for estimating the score function do have an apparent semblance to the SGD model used here (recall \eqref{scoreError}), these are far more complicated. Thus the foregoing analysis addresses to some extent an idealized situation, although it captures the spirit of the phenomenon adequately to underscore the underlying mechanisms.
There are other minor simplifications which too will contribute to our analysis not being fully accurate. We club them together with the above as sources of error. 
In reality, the memorization phenomenon  is often not so drastic and sometimes not there at all. This is where the approximation errors play a role. We shall simplify the analysis by treating approximation errors as persistent external noise. Then the situation is similar to that of Theorem \ref{collap} $(i)$ rather than that of  Theorem \ref{collap} $(ii)$ that was invoked above in our modelling of memorization. Thus the underlying phenomenon is not pure collapse, but a corruption of collapse by persistent noise. This is why when memorization occurs, it need not be exact repeated duplication of past outputs, but a bias towards noisy versions thereof to a varying degree and duration. This `degree' will depend on many parameters, which we discuss next.

It is observed in \cite{Wu} that a large step size discourages memorization. This is not a surprise because of the known fact that a large step size leads to higher levels of noise in the iteration, while maintaining a faster trend towards the goal (\cite{Borkarbook}, Chapter 9). This is evident if you consider the iteration as a noisy discretization of the corresponding ODE limit. This viewpoint treats the step size $a$ as a time step. Thus larger $a$ implies faster tracking of the ODE.  But $a$ also multiplies the noise, so it leads to higher fluctuations. This is a standard trade-off in stochastic iterative algorithms. Naturally, bigger step size leads to more noise which, as already argued, discourages memorization. .

To touch base with what we started with, the classical diffusion model takes an image as an input to an Ornstein-Uhlenbeck process run till you get `almost' noise, then take that noisy image as an input to a time-reversed diffusion to generate an image which putatively is another image from the same distribution as the one that generated the original image. The drift function of the reversed diffusion has the unknown term $\nabla\log p(X(t)|X(0),t)$ where $p(\cdot|X(0),t)$ is the probability density of $X(t)$ at time $t$, starting at $X(0)$. This is the \textit{score function} whose parametric approximation is learnt iteratively. We then have $S = C([0,T];\R^d)^2$ and the iteration \eqref{SGDor} is our surrogate for the learning algorithm for the score function. The well known memorization phenomenon in this particular context is the diffusion model generating the same image over and over again every now and then. But while the problem of diffusion model training was our starting motivation, we prefer to remain at the the level of abstraction we have been using hitherto and avoid being too specific, because diffusion models and generative models in eneral  are fast moving fields and anything ultra-specific that we say will soon be dated.

%% file: appendix.tex
\section*{Appendix B: The problem of gradient}
 The notation $\nabla_1, \nabla_2$ in \eqref{fast2}-\eqref{slow2} makes it clear that the gradients do not account for the additional dependence on $x_n,y_n$ in their role as subscripts of $\pi_{\cdots}(dz)$. Hence \eqref{fast2}-\eqref{slow2} do not necessarily constitute an instance of SGD. 
Note, however, that this observation applies only when an actual gradient is used. As already pointed out, in most cases, one uses an approximate gradient that requires only function evaluations, but no explicit differentiation. Examples are classical finite difference methods such as the Kiefer-Wolfowitz algorithm \cite{KW}, two sample simultaneous perturbation stochastic approximation (SPSA) \cite{Spall}, methods based on local regression \cite{Mukherjee},  single sample SPSA \cite{Spall} and stochastic perturbation based methods \cite{Flaxman}. However, the schemes that require two function approximations  are not free of  problems. They require iterate values at two time instants, $t$ and $t + \delta$ for some $0 < \delta \ll 1$. The finite difference approximation then requires that two separate copies of Markov noise be available: one whose transitions are modulated by $x(n)$ and another whose transitions are modulated by $x(n) +$ a small perturbation. This is not feasible in on-line training and doubles the simulation budget in off-line training. Schemes that use a single function evaluation appear to be the only ones that are free of these problems. Consider for example the noisy estimator for $\nabla F(x_n)$ given by
$$\frac{f(x_n + \delta \xi_n, Z_n)\xi_n}{\delta}..$$
%\tcr{This should be properly stated}. \\
Here $\{\xi_n\}$ are $d$-dimensional i.i.d.\ zero mean vectors with identity covariance matrix. This is adapted from the proposal of \cite{Flaxman} to accommodate Markov noise.  \\

Let $\nu$ denote the common law of  $\xi_n, n \geq 0$. 
Let $\F_n :=$ the $\sigma$-field $\sigma(x_m,Z_m, \xi_m, m \leq n),$ \ $n \geq 0$, where $\{Z_n\}$ is a Markov noise with
$$P(Z_n \in A | \F_n) = p_{x_n + \delta\xi_n}(A|Z_n)$$
for a transition kernel $p_x(\cdot | \cdot)$ as in \eqref{p-ex}. We can view $\{(Z_n,\xi_n)\}$ as an extended Markov noise. Suppose the stationary law of $(Z_n,\xi_n)$ given $x_m\equiv x$ is given by $\pi_n(dz,dy|x) = \nu(dy)\eta_{x + \delta y}(dz)$ where $\eta_x(dz)$ is the unique stationary distribution of $p_x(dz'|z)$. Then
\begin{eqnarray*}
&& E\left[\frac{f(x_n + \delta \xi_n, Z_n)\xi_n}{\delta} \ \Big| \ \F_n\right] \\
&=& \frac{1}{\delta}\int \int (f(x_n + \delta y,z)y\pi_n(dz,dy)  \\
&=& \frac{1}{\delta}\int f(x_n + \delta y,z)y\eta_{x_n + \delta y}(dz)\nu(dy) \\
&=& \frac{1}{\delta}\int \int f(x_n,z)y\eta_{x_n}(dz)\nu(dy) \\
&& + \ \frac{1}{\delta}\int\int ( \nabla f(x_n,z) \eta_{x_n}(dz))\delta yy^T\nu(dy) + o(\delta) \\
&=&  \int \nabla_x f(x_n,z)\eta_{x_n}(dz) + o(\delta) \\
&=& \nabla_x \left(\int f(x,z)\eta_{x}(dz)\right)\Big|_{x=x_n} + o(\delta) \\
&=& \nabla F(x_n) + o(\delta)
\end{eqnarray*}
for $F(\cdot) := \int f(\cdot, z)\eta_\cdot(dz).$

Here in the fourth equality, we have used the fact that $\xi_n$ has zero mean and identity covariance matrix, i.e., $\int y\nu(dy) =$ the zero vector and $\int yy^T\nu(dy) =$ the identity matrix.
Hence we have
$$\nabla F(x_n)) \approx E\left[\frac{f(x_n + \delta \xi_n, Z_n)\xi_n}{\delta} \ \Big| \ \F_n\right]$$
as desired.

An alternative scheme is based on  single function evaluation based SPSA of \cite{Spall} wherein we estimate
$\frac{\partial F}{\partial x(i)}(x_n)$ by
$$\frac{f(x_n + \delta \xi_n, Z_n)}{\delta\xi_n(i)}$$
with $\xi_n(i), 1 \leq i \leq d, n \geq0$, i.i.d.\ taking values $\pm 1$ with equal probability. This too can be analyzed along above lines.

Both cases suffer from the small divisor problem leading to high variance and need additional averaging in order to control it. They are, however, theoretically kosher unlike \eqref{fast2}-\eqref{slow2} and also avoid the complications of schemes that require two or more samples.